\documentclass[letterpaper, 10 pt, conference]{ieeeconf}

\IEEEoverridecommandlockouts 

\usepackage{amsfonts}
\usepackage{graphics}
\usepackage{epsfig}
\usepackage{amsmath}
\usepackage{amssymb}
\usepackage{tikz}
\usepackage{tikz-cd}
\usetikzlibrary{patterns,quotes,angles}

\newtheorem{definition}{Definition}

\newtheorem{remark}{Remark}
\newtheorem{lemma}{Lemma}
\newtheorem{theorem}{Theorem}
\newtheorem{proposition}{Proposition}
\let\myemph\emph
\renewcommand{\emph}[1]{{\bfseries\itshape{#1}}}

\newcommand{\X}{\mathcal{X}}

\newcommand*{\dd}{\mathrm{d}}

\newcommand*{\contr}[1]{\iota_{#1}}

\newcommand{\norm}[1]{\left\lvert\left\lvert #1 \right\rvert\right\rvert}

\makeatletter
\renewcommand*{\@opargbegintheorem}[3]{\trivlist
  \item[\hskip \labelsep{\bfseries #1\ #2}] \textbf{(#3)}\ \itshape}
\makeatother

\title{\LARGE \bf
The interplay between symmetries and impact effects\\ on hybrid mechanical systems*
}

\author{William Clark$^{1}$, Leonardo Colombo$^{2}$, and Anthony Bloch$^{3}$
\thanks{*W. Clark was funded by AFOSR grant FA9550-23-1-0400. L. Colombo acknowledges financial support from Grant PID2022-137909NB-C21 funded by MCIN/AEI/ 10.13039/501100011033. A.Bloch was partially supported by NSF grant  DMS-2103026, and AFOSR grants FA
9550-22-1-0215 and FA 9550-23-1-0400.}
\thanks{$^{1}$William Clark is with the Department of Mathematics,
        Ohio University, Athens, OH, 45701, USA
        {\tt\small clarkw3@ohio.edu}}%
\thanks{$^{2}$L. Colombo is with Centre for Automation and Robotics (CSIC-UPM), Ctra. M300 Campo Real, Km 0,200, Arganda del Rey - 28500 Madrid, Spain. {\tt\small leonardo.colombo@csic.es}}
\thanks{$^{3}$A. Bloch is with Department of Mathematics, University of Michigan, Ann Arbor, MI 48109, USA. {\tt\small abloch@umich.edu}}
}

\begin{document}

\maketitle
\thispagestyle{empty}
\pagestyle{empty}

\begin{abstract}

Hybrid systems are dynamical systems with continuous-time and discrete-time components in their dynamics. When hybrid systems are defined on a principal bundle we are able to define two classes of impacts for the discrete-time transition of the dynamics:  interior impacts and exterior impacts. In this paper we define hybrid systems on principal bundles, study the underlying geometry on the switching surface where impacts occur and we find conditions for which both exterior and interior impacts are preserved by the mechanical connection induced in the principal bundle. 

\end{abstract}

\section{Intoduction}
Hybrid systems are dynamical systems with continuous-time and discrete-time components in their dynamics. These dynamical systems  are capable of modeling various physical systems, such as multiple UAV systems \cite{lee_geometric_2013}, bipedal robots \cite{westervelt_feedback_2018} and embedded computer systems \cite{goebel_rafal_hybrid_2012}, \cite{schaft_introduction_2000}, among others. 

Simple hybrid systems are a type of hybrid systems introduced in \cite{johnson_simple_1994} denoted as such because of their simple nature. A simple hybrid system is characterized by a tuple $\mathcal{H}=(\X,\mathcal{S},X,\Delta)$ where $\X$ is a smooth manifold, $X$ is a smooth vector field on $\X$, $\mathcal{S}$ is an embedded submanifold of $\X$ with co-dimension $1$ called the switching surface (or the guard), and $\Delta:\mathcal{S}\to \X$ is a smooth embedding called the impact map (or the reset map). This type of hybrid system has been mainly employed for the understanding of walking gaits in bipeds and insects \cite{westervelt_feedback_2018}, \cite{ames_geometric_2007}, \cite{holmes_dynamics_2006}. In the situation where the vector field $X$ is associated with a mechanical system (Lagrangian or Hamiltonian), alternative approaches for unilateral constraints and hybrid systems with symmetries have been considered in \cite{bloch2017quasivelocities}, \cite{colombo2020symmetries}, \cite{cortes_hamiltonian_2006}, \cite{cortes_mechanical_2001}, \cite{ibort_geometric_1998}, \cite{ibort_mechanical_1997}.

We consider here the role of connections in understanding hybird systems on manifolds. Roughly speaking, a connection tells us how a quantity
associated with a manifold changes as we move from one
point to another - it “connects” neighboring spaces. In terms
of fiber bundles, a connection tells us how movement in the
total space induces change along the fiber. An important connection for analyzing the dynamics and control of mechanical systems is the mechanical connection which is defined in terms of the momentum map associated with a Lie group of symmetries. 

The following motivating example shows the interplay between symmetries, mechanical connection on principal bundles and the role of impacts by studying the underlying geometry of the switching surface and preservation properties of the impact map.

\subsection{Motivating example: pendulum on the cart}

We begin with a case study: the pendulum on a cart. Details on the mathematical background are given 
below. The configuration space for this system  (see Fig. \ref{fig:pendulum_cart}) is $Q = \mathbb{S}^1\times\mathbb{R}$ with Lagrangian
\begin{align*}
    L(\theta,x,\dot{\theta},\dot{x}) &= \frac{1}{2}\left( m\ell^2\dot{\theta}^2 + 2m\ell\dot{x}\dot{\theta}\cos\theta + (M+m)\dot{x}^2\right)\\&\quad + mg\ell\cos\theta.
\end{align*}
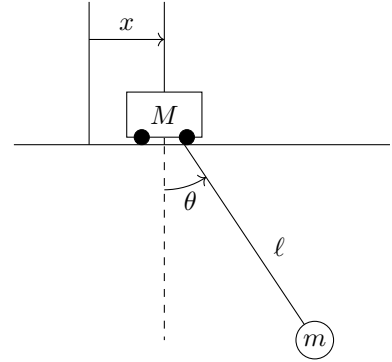
\begin{figure}[h!]
	\centering
	\begin{tikzpicture}
		\draw (-2,-0.4) -- (3,-0.4);
		\draw (0,-0.4) -- (0,1.5);
		\draw (0,0) -- (2,-3);
		\draw [dashed] (0,0) -- (0,-3);
		\draw [fill=white] (2,-3) circle [radius=0.25];
		\node at (2,-3) {$m$};
		\draw [fill=white] (-0.5,-0.3) rectangle (0.5,0.3);
		\node at (0,0) {$M$};
		\draw [fill] (-0.3,-0.3) circle [radius=0.1];
		\draw [fill] (0.3,-0.3) circle [radius=0.1];
		\path  (0,-3) coordinate (a)
            -- (0,0) coordinate (b)
            -- (2,-3) coordinate (c)
        pic["$\theta$", draw=black, ->, angle eccentricity=1.2, angle radius=1cm] {angle=a--b--c};
		\draw (-1,-0.4) -- (-1,1.5);
		\draw [->] (-1,1) -- (0,1);
		\node [above] at (-0.5,1) {$x$};
		\node [above right] at (1.33,-2) {$\ell$};
	\end{tikzpicture}
	\caption{The pendulum on the cart.}
	\label{fig:pendulum_cart}
\end{figure}
\newline
The underlying metric is
\begin{align*}
    g &= m\ell^2 d\theta\otimes d\theta + m\ell\cos\theta\left( dx\otimes d\theta + d\theta\otimes dx\right) \\&+ (M+m) dx\otimes dx.
\end{align*}
If we let $s\in G=\mathbb{R}$ and denote its action on $Q$ by $s.(\theta,x) = (\theta,x+s)$, then the Lagrangian is invariant under the tangent lift of left translations, that is, $L(\theta,x,\dot\theta,\dot x)=L(\theta, x+s, \dot{\theta},\dot{x})$. The configuration space has the structure of a principal $\mathbb{R}$-bundle,
\begin{gather*}
    \pi:Q = \mathbb{S}^1\times\mathbb{R} \to\mathbb{S}^1 \\
    (\theta,x) \mapsto \theta
\end{gather*}
and the vertical space is
\begin{equation*}
    \mathcal{V}_{(\theta,x)} = \ker T\pi_{(\theta,x)} = \mathrm{span}_\mathbb{R}\left( \frac{\partial}{\partial x}\right) \subset T_{(\theta,x)}Q.
\end{equation*}
The horizontal space arises from the mechanical connection, $\mathcal{A}:TQ\to\mathbb{R}$
\begin{equation*}
    \mathcal{A}_{(\theta,x)}(\dot\theta,\dot{x}) = \frac{1}{M+m} \left( (M+m)\dot{x} + m\ell\dot\theta\cos\theta\right),
\end{equation*}
as the locked inertia tensor (see \cite{BL} for instance) $\mathbb{I}:\mathbb{R}\to\mathbb{R}$ is simply $\mathbb{I} = (M+m)$. The resulting horizontal space is 
\begin{equation*}
    \overline{\mathcal{H}}_{(\theta,x)} = \ker \mathcal{A}_{(\theta,x)} = \mathrm{span}_\mathbb{R} \left( \frac{\partial}{\partial\theta} - \frac{m\ell\cos\theta}{M+m}\frac{\partial}{\partial x} \right).
\end{equation*}
We note that the momenta (which will be useful later) are
\begin{equation*}
    \begin{split}
        p_\theta &= m\ell^2\dot\theta + m\ell\dot{x}\cos\theta, \\
        p_x &= m\ell\dot\theta\cos\theta + (M+m)\dot{x}.
    \end{split}
\end{equation*}
Impacts can be imposed on this system in two qualitatively distinct ways depending on the geometry of the guard: interior and exterior impacts. As an abuse of notation, we say that the guard $\mathcal{S}\subset Q$ rather than $\mathcal{S}\subset \X = TQ$.
\subsubsection{Interior Impacts}
Suppose that the pendulum on the cart impacts at the location $\theta = \alpha$. In this case, the impact surface is $\mathcal{S} = \{\alpha\}\times\mathbb{R}\subset Q$. Notice that the impact surface is a lift over the base space, $\mathcal{S} = \pi^{-1}(\alpha)$. In this case, the (elastic) impact map is given by
\begin{equation}\label{eq:outer_cart_first}
    \tilde{\Delta}: \begin{cases}
    \begin{aligned}
        p_\theta &\mapsto -p_\theta + \frac{2m\ell}{M+m} p_x\cos\theta,\\
        p_x &\mapsto p_x.
    \end{aligned} \end{cases}
\end{equation}
 In velocity coordinates, the impact map is
\begin{equation}\label{eq:inner_cart}
    \Delta: \begin{cases}
        \begin{aligned}
            \dot{\theta} &\mapsto -\dot{\theta} \\
            \dot{x} &\mapsto \dot{x} + \frac{2m\ell}{M+m}\dot{\theta}\cos\theta
        \end{aligned}
    \end{cases}
\end{equation}
In particular, the mechanical connection is conserved: $\Delta^*\mathcal{A} = \mathcal{A}|_\mathcal{S}$.
\subsubsection{Exterior Impacts}
Suppose that the cart impacts a wall at the location $x = z$. In this case, the impact surface is $\mathcal{S} = \mathbb{S}^1\times\{z\}$. Unlike before, the surface is no longer a lift of something in the base space. In this case, the (elastic) impact map is given by
\begin{equation*}
    \tilde{\Delta}:\begin{cases}
        \begin{aligned}
            p_\theta &\mapsto p_\theta \\
            p_x &\mapsto -p_x+ \frac{2}{\ell}p_\theta\cos\theta
        \end{aligned}
    \end{cases}
\end{equation*}
 In velocity coordinates, the impact map is
\begin{equation}\label{eq:outer_cart}
    \Delta: \begin{cases}
        \begin{aligned}
            \dot{\theta} &\mapsto \dot{\theta} + \frac{2}{\ell}\dot{x}\cos\theta \\
            \dot{x} &\mapsto -\dot{x}
        \end{aligned}
    \end{cases}
\end{equation}

Unlike  the interior case, the mechanical connection is no longer preserved: $\Delta^*\mathcal{A}\ne\mathcal{A}$.

\begin{remark}
    Preservation of the connection can be advantageous as it can allow one to reduce the system. Likewise, breaking this symmetry has the advantage of introducing an added level of controllability to the system.
\end{remark}

The \textbf{problem} we are interested in studying in this paper consists of finding conditions for which both exterior and interior impacts are preserved by the mechanical connection $\mathcal{A}$ and studying in general hybrid systems on principal bundles.

We note that while this work is concerned with controls and the distinction between interior and exterior impacts the 
related work \cite {colombo2021generalized} focuses on reduction by symmetries in both, the continuous-time dynamics and the impact dynamics.

\subsection{Structure of the paper:}  In Section \ref{subsection_lagrangian_systems} we review the geometric formalism for mechanical systems on differentiable manifolds, in particular, for Lie groups, and the Weierstrass-Erdmann corner conditions. Section \ref{sec3} introduces hybrid mechanical systems and their conserved quantities together with a discrete Noether theorem for hybrid systems. In Section \ref{sec4} we define hybrid systems on principal bundles and we describe the subjacent geometry on the switching surface. Finally in Section \ref{sec5} we introduce exterior and interior impacts and study under which conditions the mechanical connection is preserved across impacts.
Conclusions remarks and thoughts about future directions close the paper.
\section{Preliminaries on the geometric formulation of mechanical systems} \label{subsection_lagrangian_systems}

This section introduces conventional mathematical notions to describe simple mechanical systems on differentiable manifolds which can be found in \cite{BL} and \cite{mars}, for instance. 

Let $Q$ be a differentiable manifold with $\hbox{dim}(Q)=n$ with local coordinates denoted by $q^i$. Its tangent and cotangent bundles are given by $TQ$ and $T^*Q$ with induced coordinates $(q^i,\dot{q}^i)$ and $(q^i,p_i)$ respectively. The tangent and cotangent spaces at a point $q$ are denoted by $T_qQ$ and $T_q^*Q$.

The dynamics of a mechanical system can be determined by the Euler-Lagrange equations associated with a Lagrangian function $L:TQ\to\mathbb{R}$. A \textit{mechanical Lagrangian} is given by $L(q,\dot{q})=K(q,\dot{q})-V(q),$ where $K:TQ\to\mathbb{R}$ is the kinetic energy and $V:Q\to\mathbb{R}$ the potential energy. The kinetic energy is given by $K(q, \dot q)=\frac{1}{2}\norm{\dot q}^2_q$, where $\norm{\cdot}_q$ denotes the norm at $T_qQ$ defined by some (pseudo)Riemannian metric on $Q$. In particular, a mechanical Lagrangian will be called \textit{kinetic} if $V=0$.

The equations describing the dynamics of the system are given by the Euler-Lagrange equations 
$$\displaystyle{\frac{\dd }{\dd  t}\left(\frac{\partial L}{\partial\dot{q}^{i}}\right)=\frac{\partial L}{\partial q^{i}}},$$ 
with $i=1,\ldots,n;$ a system of $n$ second-order ordinary differential equations. 

Let $M$ and $N$ be smooth manifolds. For each $p$-form $\alpha$ and each vector field $X$ on $M$, $\contr{X} \alpha$ denotes the interior product of $\alpha$ by $X$. For a smooth map $F:M\to N$, its tangent map $TF:TM\to TN$ will  be called its pushforward. 
The pullback of differential forms by this map will be denoted by $F^{*}$. Unless otherwise stated, sum over paired covariant and contravariant indices will be understood.

We denote $\mathbb{F}L\colon TQ\to T^*Q$ the Legendre transform associated with $L$; in the case of a mechanical Lagrangian, $\mathbb{F}L(u)(v) = \langle u, v\rangle_q$ and is a diffeomorphism. This transformation relates the Lagrangian and Hamiltonian formalisms. Define the Hamiltonian function $H:T^*Q\to\mathbb{R}$ as $H(q,p) = p^T \dot q(q,p)- L(q,\dot q(q,p))$, where we have used the inverse of the Legendre transformation to express $\dot q=\dot q(q,p)$. 
Trajectories obey Hamilton's equations
\begin{equation*}
    \dot{q}^i = \frac{\partial H}{\partial p_i}, \quad \dot{p}_i = -\frac{\partial H}{\partial q^i}.
\end{equation*}
A Hamiltonian is said to be \textit{mechanical} if its associated Lagrangian is mechanical.

\subsection{Variational Corner Conditions}
As the continuous dynamics (Euler-Lagrange) are \textit{variational}, we will define the impact to be variational as well. 
This is realized by the Weierstrass-Erdmann corner conditions (see \S 4.4 of \cite{kirk2004optimal} or \S 3.5 of \cite{brogliato1999nonsmooth}). Suppose that at some undetermined time, $t^*$, an impact occurs $q(t^*)\in S$. Variations in the impact time lead to energy conservation while variations in the impact location result in momentum conservation along the surface, i.e.
\begin{equation*}
    \mathbb{F}L\cdot \delta q \bigg|_{t^-}^{t^+} + \left( L - \langle \mathbb{F}L, \dot{q}\rangle\right)\cdot \delta t\bigg|_{t^-}^{t^+} = 0.
\end{equation*}
As $\delta q\in TS$ and $\delta t\in\mathbb{R}$ is arbitrary, the above condition can be written as:
\begin{equation}\label{eq:WEL}
    \begin{split}
        \mathbb{F}L^+ - \mathbb{F}L^- &= \alpha \cdot dh,\\
        L^+ - \langle \mathbb{F}L^+,\dot{q}^+\rangle &= L^- - \langle \mathbb{F}L^-,\dot{q}^-\rangle,
    \end{split}
\end{equation}
where $S = \left\{ q\in Q : h(q) = 0\right\}$ is given by the level set of a smooth function $h:Q\to\mathbb{R}$, and the multiplier $\alpha$ is set so both equations are satisfied. 
The superscripts denote the values immediately pre- and post-impact, e.g.
\begin{equation*}
    \begin{split}
        \mathbb{F}L^+ &= \lim_{t\searrow t^*}\, \mathbb{F}L(q(t),\dot{q}(t)), \\
        \mathbb{F}L^- &= \lim_{t\nearrow t^*}\, \mathbb{F}L(q(t),\dot{q}(t)).
    \end{split}
\end{equation*}
These corner conditions have a clearer interpretation in the Hamiltonian setting:
\begin{equation}\label{eq:pH}
\begin{split}
p^+ &= p^- + \alpha\cdot dh, \\
H^+ &= H^-.
\end{split}
\end{equation}
i.e. energy at impacts is conserved and the change in momentum is perpendicular to the impact surface which is precisely specular reflection.
\subsection{Lie group actions}
Let $G$ be a finite dimensional Lie group and $Q$ a smooth manifold. A \textit{left-action} of $G$ on $Q$ is a smooth map $\psi:G\times Q\to Q$ such that $\psi(e,g)=g$ and $\psi(h,\psi(g,q))=\psi(hg,q)$ for all $g,h\in G$ and $q\in Q$, where $e$ is the identity of the group $G$ and the map $\psi_g:Q\to Q$ given by $\psi_g(q)=\psi(g,q)$ is a diffeomorphism for all $g\in G$.

For a Lie group $G$, let $\mathfrak{g}$ be its Lie algebra, $\mathfrak{g}:=T_{e}G$.
Let $\mathcal{L}_{g}:G\to G$ be the left-translation of the element $g\in G$ given by $\mathcal{L}_{g}(h)=gh$ for $h\in G$. Left-translation is a left-action of $G$ on itself \cite{HolSch09}.
Its tangent map  (i.e, the linearization or tangent lift) is denoted by $T_{h}\mathcal{L}_{g}:T_{h}G\to T_{gh}G$. Similarly, the cotangent map (cotangent lift) is denoted by $T_{h}^{*}\mathcal{L}_{g}: T^{*}_{h}G\to T^{*}_{gh}G$. It is well known that the tangent and cotangent lifts are Lie group actions (see \cite{HolSch09}, Chapter $6$).

A Lagrangian function $L:TG\to\mathbb{R}$ is said to be invariant under the tangent lift of left translations if $L(g,\dot{g})=L(\mathcal{L}_hg,T_g\mathcal{L}_h\dot{g})$
 
Let $\psi\colon G\times Q\to Q$ be a Lie group action. $\psi$ is said to be a \textit{free} action if it has no fixed points, that is, $\psi_g(q)=q$ implies $g=e$. The Lie group action $\psi$ is said to be a \textit{proper} action if the map
$\tilde{\psi}:G\times Q\to Q\times Q$ given by $\tilde{\psi}(g,q)=(q,\psi(g,q))$, is proper, that is, if $K\subset Q\times Q$ is compact, then $\tilde{\psi}^{-1}(K)$ is compact.

\subsection{Principal Bundles}
The action of a group on a manifold leads to the notion of a principal bundle.
\begin{definition}
    A fiber bundle is a triple $(E,\pi,M)$ where $\pi:E\to M$ is a surjective map with the property that for all $m\in M$, there exists an open neighborhood $m\in U\subset M$ such that there is a diffeomorphism
    \begin{equation*}
        \pi^{-1}(U)\cong U\times F,
    \end{equation*}
    where $F$ is said to be the fiber of the bundle.
\end{definition}
In the special case where the fiber is a Lie group, additional structure can be imposed on the fiber bundle.
\begin{definition}
    Let $(E,\pi,M)$ be a fiber bundle with fiber $G$ (a Lie group). This is a principal $G$-bundle if $G$ acts on $E$ such that it preserves the fibers and acts freely and transitively on each fiber.
\end{definition}
Let $\Phi:G\times Q\to Q,$ $(g,q)\mapsto\Phi_{g}(q)$ be a free and proper left action of a Lie group $G$ on a manifold $Q.$ Thus we can define the principal bundle $\pi:Q\to M:= Q/G,$ where $M$ is endowed with the unique manifold structure for which $\pi$ is a submersion (see \cite{kobayashi1996foundations}). $M:=Q/G$ is called the \textit{shape space} in mechanics.

\subsection{Momentum Maps and the Mechanical Connection}
Momentum maps capture in a geometric way conserved quantities associated with symmetries. The momentum map is related to the so-called mechanical connection by the locked inertia tensor as defined below. Roughly speaking, a \textit{connection} tells us how a quantity associated with a manifold changes as we move from one point to another - it ``connects'' neighboring spaces. In terms of fiber bundles, a connection tells us how movement in the total space induces change along the fiber.

Let $G$ be a finite-dimensional Lie group acting on the cotangent bundle by the cotangent lift of left translations. Denote the corresponding infinitesimal action of $\mathfrak{g}$ on $T^{*}Q$ by $\xi\mapsto\xi_{T^{*}Q}$, a map of $\mathfrak{g}$ to $\mathfrak{X}(T^{*}Q)$, the space of vector fields on $T^{*}Q$. We write the action of $g\in G$ on $z\in T^{*}Q$ as simply $gz$, the vector field $\xi_{T^{*}G}$ is obtained at $z$ by differentiating $gz$ with respect to $g$ in the direction of $\xi$ at $g=e$. Explicitly,  $$\xi_{T^{*}G}(z)=\frac{\dd}{\dd\epsilon}(\exp(\epsilon\xi)\cdot z)\Big{|}_{\epsilon=0}.$$

A map $\textbf{J}:T^{*}Q\to\mathfrak{g}$ is called a \textit{momentum map} if $X_{\langle\mathbf{J},\xi\rangle}=\xi_{T^{*}G}$ for each $\xi\in\mathfrak{g}$, where $\langle\mathbf{J},\xi\rangle(z)=\langle\mathbf{J}(z),\xi\rangle$.
Noether's theorem states that if $H$ is a $G$-invariant Hamiltonian function on $T^{*}G$ then $\mathbf{J}$ is conserved on trajectories of the Hamiltonian vector field $X_H$.

The momentum map is closely related to the \textit{mechanical connection}. Let $\langle\langle\cdot, \cdot\rangle\rangle$ be the group-invariant metric induced by the invariant Lagrangian.
For each $q\in Q$ define the \textit{locked inertia tensor} to be the map $\mathbb{I}(q):\mathfrak{g}\to\mathfrak{g}^{*}$ defined by $$\langle\mathbb{I}(q)\eta,\zeta\rangle=\langle\langle\eta_{Q}(q),\zeta_{Q}(q)\rangle\rangle.$$ We define the \textit{mechanical connection} on the principal bundle $Q\to Q/G$ to be the map $\mathcal{A}_s:TQ\to\mathfrak{g}$ given by $$\mathcal{A}_s(q,v)=\mathbb{I}(q)^{-1}(\mathbf{J}(q,v));$$ that is, $\mathcal{A}_s$ is the map that assigns to each $(q,v)$ the corresponding angular velocity of the locked system.

One can check that $\mathcal{A}_s$ is $G$-invariant and $\mathcal{A}_s(\xi_{Q}(q))=\xi$, and the horizontal space of the connection is given by $\overline{\mathcal{H}}_q=\{(q,v)|\mathbf{J}=0\}\subset T_qQ$ and the vertical space is given by $\mathcal{V}_q=\{\xi_{Q}(q)|\xi\in\mathfrak{g}\}$.

\section{Hybrid mechanical systems}\label{sec3}
Hybrid dynamical systems are dynamical systems characterized by their mixed behavior of continuous and discrete dynamics where the transition is determined by the time when the continuous flow switches from the ambient space to a co-dimension one submanifold. This class of dynamical systems is given by an $4$-tuple, $\mathcal{H}=(\X,\mathcal{S},z,\Delta)$. The pair ($\X$,$z$) describes the continuous dynamics as 
$\dot{x}(t) = z(x(t))$,
where $\X$ is a smooth manifold and $z$ a $C^1$ vector field on $\X$. Additionally, ($\mathcal{S}$, $\Delta$) describes the discrete dynamics as 
$x^{+}=\Delta(x^{-})$ where $\mathcal{S}\subset\X$ is a smooth submanifold of co-dimension one called the \textit{switching surface}.

The hybrid dynamical system describing the combination of both dynamics is given by
\begin{equation}\label{sigma}\Sigma: \begin{cases}
\dot{x} = z(x),& x\not\in \mathcal{S}\\
x^+ = \Delta(x^-),& x^-\in \mathcal{S}.
\end{cases}\end{equation}

A solution of a hybrid dynamical system may experience a Zeno state if infinitely many impacts occur in a finite amount of time. To exclude these types of situations, we require the set of impact times to be closed and discrete, as in \cite{westervelt_feedback_2018}, so we will assume implicitly throughout the remainder of the paper that $\overline{\Delta}({S})\cap{S}=\emptyset$ (where $\overline{\Delta}({S})$ denotes the closure of $\Delta({S})$) and that the set of impact times is closed and discrete.

\begin{definition}
	A simple hybrid system $\mathcal{H}=(\mathcal{X}, \mathcal{S}, z, \Delta)$ is said to be a \textit{simple hybrid Lagrangian system} if it is determined by $\mathcal{H}_{L}:= (TQ, {\mathcal{S}_L}, X_{L}, \Delta_{L})$, where $X_{L}$ is the Lagrangian vector field associated with the Lagrangian system determined by $L$, ${\mathcal{S}_L}$ is the switching surface, a submanifold of $TQ$ with co-dimension one, and $\Delta_{L}:{\mathcal{S}}_L\to TQ$ is the impact map described by the variational corner conditions \eqref{eq:WEL}, which is a smooth embedding.

  The simple hybrid Lagrangian system generated by $\mathcal{H}_{L}$ is given by
  \begin{equation}
    \label{RHDS}\Sigma_{L}:
    \left\{\begin{array}{ll}\dot{\upsilon}(t)=X_{L}(\upsilon(t)), & \hbox{ if } \upsilon^{-}(t)\notin{\mathcal{S}_L},\\ \upsilon^{+}(t)=\Delta_{L}(\upsilon^{-}(t)),&\hbox{ if } \upsilon^-(t)\in{\mathcal{S}_L}, 
    \end{array}\right.
  \end{equation}
  where $\upsilon(t)=(q(t),\dot{q}(t))\in TQ $.
\end{definition}

In a similar fashion, one can define \textit{simple hybrid Hamiltonian systems} associated with a Hamiltonian function $H:T^{*}Q\to\mathbb{R}$ through the $4$-tuple $\mathcal{H}_H=(T^{*}Q, \mathcal{S}_{H},X_{H},\Delta_{H})$ where $X_{H}$ is the Hamiltonian vector field associated with the Hamiltonian system determined by $H$, ${S_{H}}$ is the switching surface, a submanifold of $T^{*}Q$ with co-dimension one, and the smooth embedding $\Delta_{H}:{\mathcal{S}}_{H}\to T^{*}Q$ is the impact map given by the variational corner conditions \eqref{eq:pH}.

\subsection{Hybrid Noether theorem}
\begin{definition}[\cite{colombo2021generalized}] Let $(\X,S,z,\Delta)$ be a hybrid dynamical system.
    A function $f$ on $\mathcal X$ is called a \textit{hybrid constant of the motion} if it is preserved by the hybrid flow, namely, $f \circ \varphi_t^\mathcal{H} = f$. In other words,
    $z(f)=0$ and $ f \circ \Delta = f \circ i$, where $i\colon S \hookrightarrow \X$ is the canonical inclusion.
\end{definition}

There is a natural lift $\psi^{T^{*}Q}$ of the action $\psi$ to $T^{*}Q$, the \textit{cotangent lift}, defined by $(g,(q,p))\mapsto (T^{*}\psi_{g^{-1}}(q,p))$.  By a \textit{hybrid action} on the simple hybrid Hamiltonian system $\mathcal{H}_H$  we mean a Lie group  action $\psi\colon G\times Q\to Q$ such that
\begin{itemize}
	\item $H$ is invariant under $\psi^{T^{*}Q}$, i.e. $H\circ \psi^{T^{*}Q}=H$,
	\item $\psi^{T^{*}Q}$ restricts to an action of $G$ on $S_{H}$,
	\item $\Delta_{H}$ is equivariant with respect to the previous action, namely $$\Delta_{H}\circ \psi^{T^{*}Q}_g\mid_{S_{H}}=\psi^{T^{*}Q}_g\circ \Delta_{H}.$$
\end{itemize}

\begin{definition}[\cite{colombo2021generalized}]
A momentum map $\mathbf{J}$ will be called a \textit{generalized hybrid momentum map} for $\mathcal{H}_H$ if, for each regular value $\mu_-$ of $\mathbf{J}$,
\begin{equation}
  \Delta_{H} \left(\mathbf{J}|_{S_H}^{-1}(\mu_-)  \right) \subset \mathbf{J}^{-1}(\mu_+),
  \label{generalized_hybrid_momentum}
\end{equation}
for some regular value $\mu_+$. In other words, for every point in the switching surface such that the momentum before the impact takes a value of $\mu_-$, the momentum will take a value $\mu_+$ after the impact. That is, the switching map  translates the dynamics from one level set of the momentum map into another.
In particular, when $\mu_+=\mu_-$ for each $\mu^-$ (i.e., $\Delta_{H}$ preserves the momentum map), $\mathbf{J}$ is called \myemph{hybrid momentum map} (see \cite{ames_hybrid_2006}). \end{definition}

Given an action in the Lie algebra such that it preserves the Hamiltonian function and is equivariant with respect to the impact map. The \textit{hybrid Noether theorem} states that for all $\xi\in\mathfrak{g}$, the generalized momentum map $\mathbf{J}^{\xi}$ is a hybrid constant of the motion.

\section{Hybrid Mechanical Systems on Principal Bundles}\label{sec4}
As an impact system has the added structure of the impact surface, the corresponding principal bundle requires more structure as well. In addition of the surface, $\mathcal{S}$, a choice of metric is also needed for the corner conditions \eqref{eq:WEL}.
\begin{definition}
    A \textit{$G$-impact system} is a tuple\newline $(E,M,\pi,\mathcal{S},L)$ where
    \begin{enumerate}
        \item $\pi:E\to M$ is a $G$-principal bundle,
        \item $\mathcal{S}\subset E$ is an embedded, codimension 1 submanifold,
        \item $L:TE\to\mathbb{R}$ is a  mechanical Lagrangian invariant under the tangent lift of left translations.
    \end{enumerate}
\end{definition}
Impacts take place on $TE|_\mathcal{S}$ (on the Lagrangian side) or $T^*E|_\mathcal{S}$ (on the Hamiltonian side). Strictly speaking, these sets are the guards, rather than $\mathcal{S}$.
These two maps are related by the fiber derivative:
\begin{equation*}
    \begin{tikzcd}
        TE|_\mathcal{S} \arrow[r, "\Delta"] \arrow[d, "\mathbb{F}L"] & TE \arrow[d, "\mathbb{F}L"] \\
        T^*E|_\mathcal{S} \arrow[r, "\tilde{\Delta}"] & T^*E
    \end{tikzcd}
\end{equation*}
Throughout, $\Delta$ will represent the impact map on the velocities while $\tilde{\Delta}$ will be the impact map on the momenta.

\begin{definition}
    The impact surface, $\mathcal{S}$, is vertical if $\mathcal{S}=\pi^{-1}(\Sigma)$ for some embedded, codimension 1 submanifold $\Sigma\subset M$. The impact surface $\mathcal{S}$ is horizontal if $(T\mathcal{S})^\perp\subset\mathcal{V}$. Here, $\mathcal{V}\subset TE$ is the vertical space given by
    \begin{equation*}
        \mathcal{V}_x = \left\{ v\in T_xE : T\pi(v) = 0\right\}.
    \end{equation*}
\end{definition}

\begin{lemma}\label{lem:vertical_invariant}
    $\mathcal{S}$ is vertical if and only if $\mathcal{S}$ is invariant under the group action of left translations, i.e. $h.\mathcal{S} = \mathcal{S}$ for all $h\in G$.
\end{lemma}
\begin{proof}
    If $\mathcal{S}$ is vertical, it is clearly invariant under the group action of left translations (as the group action preserve fibers). Suppose that $\mathcal{S}$ is invariant under the group action then $\mathcal{S}$ must be vertical as the group acts transitively on each fiber.
\end{proof}

\begin{proposition}
    The variational corner conditions, \eqref{eq:pH}, are equivalent to 
    \begin{equation}\label{eq:impact}
         \left( \mathrm{Id}\times\tilde{\Delta}\right)^*\vartheta_H = i^*\vartheta_H,
    \end{equation}
    \begin{equation}\label{eq:impact_diagram}
        \begin{tikzcd}
             \mathbb{R}\times T^*E|_\mathcal{S} \arrow[rr, "\mathrm{Id}\times\tilde{\Delta}"] \arrow[dr, "\mathrm{Id}\times\pi" below left] && \mathbb{R}\times T^*E \arrow[dl, "\mathrm{Id}\times\pi"] \\
             &\mathbb{R}\times E&
         \end{tikzcd}
    \end{equation}
    such that the diagram is commutative. Here, $\vartheta_H = p_i\cdot dq^i - H\cdot dt \in \Omega^1(\mathbb{R}\times T^*E)$ is the action form
\end{proposition}
\begin{proof}
    Suppose that $\mathcal{S}$ is given (locally) by the vanishing of the last coordinate, $q^n = 0$. Then, \eqref{eq:impact} in coordinates is equivalent to
    \begin{align*}
        \begin{cases}
            H^+dt^+ = H^-dt^-, \\
            p_1^+{dq^1}^+ + \ldots + p_{n-1}^+ {dq^{n-1}}^+ = \\ \hspace{3cm}
            p_1^-{dq^1}^- + \ldots + p_{n-1}^- {dq^{n-1}}^-.
        \end{cases}
    \end{align*}
    Commutativity of \eqref{eq:impact_diagram} means that ${q^k}^+ = {q^k}^-$ and $t^+ = t^-$. These conditions are precisely \eqref{eq:pH} as $dh = dq^{n}$.
\end{proof}
\begin{proposition}
    The variational corner conditions for a mechanical Lagrangian with Riemannian metric $g$ are equivalent to $\Delta^*\omega = \omega|_{\mathcal{S}}$ and $\Delta^*g = g|_{\mathcal{S}}$
    (along with the positions being fixed, $q^+ = q^-$)
    where $\omega\in\Omega^2(TE)$ is the pull-back of the canonical symplectic form on $T^*E$ via the metric. 
\end{proposition}
\begin{proof}
    Again, suppose that $\mathcal{S}$ is described by $q^n=0$. 
    The first equality states that 
    \begin{equation*}
        \sum_{k=1}^{n-1} \, {dq^k}^+\wedge dp_k^+ = 
        \sum_{k=1}^{n-1} \, {dq^k}^-\wedge dp_k^-
    \end{equation*}
    This means that the only momentum that can change during impacts is $p_n$. The second equality follows from conservation of (kinetic) energy.
\end{proof}
\section{Interior vs Exterior Impacts}\label{sec5}
An exterior impact preserves the inner dynamics while an internal impact preserves the outer dynamics. 
\begin{definition}
    An impact is interior if the impact surface is vertical. An impact is exterior if $\pi(\mathcal{S}) = M$, the whole shape space.
\end{definition}
Notice that the impact surface being horizontal is more restrictive than the impact being exterior.
Properties of vertical/horizontal impact surfaces are shown in the following two theorems.
\begin{theorem}\label{thm:hor}
    The impact surface being horizontal is equivalent to $T\pi\circ\Delta = T\pi$.
\end{theorem}
\begin{proof}
    As $\mathcal{S}$ is horizontal, we have $(T\mathcal{S})^\perp\subset \mathcal{V}$. This implies that $v^+ - v^-\in\mathcal{V}$, i.e. $T\pi(v^+) = T\pi(v^-)$.
\end{proof}
\begin{remark}
    The statement that $T\pi\circ\Delta= T\pi$ means that the impact reduces to the identity on the shape space. As such, an impact is unobservable from dynamics on the shape variables. This condition is equivalent to the following commutative diagram:
    \begin{equation}\label{diag:horizontal}
        \begin{tikzcd}
            & TE|_\mathcal{S} \arrow[dl, "\pi_*" above left] \arrow[dd, "\Delta" left] \arrow [dr] & & \\
            TM & & E \arrow[r, "\pi"] & M \\
            & TE\arrow[ul, "\pi_*"] \arrow[ur] & & 
        \end{tikzcd}
    \end{equation}
\end{remark}
\begin{theorem}
    The following are equivalent:
    \begin{enumerate}
        \item $\mathcal{S}$ is vertical,
        \item $\mathcal{S}$ is $G$-invariant, i.e. for any $h\in G$, $h.\mathcal{S}=\mathcal{S}$,
        \item The mechanical connection is preserved across impacts, $\Delta^*\mathcal{A} = \mathcal{A}|_\mathcal{S}$.
    \end{enumerate}
\end{theorem}
\begin{proof}
    The equivalence of (1) and (2) follows from Lemma \ref{lem:vertical_invariant}. The equivalence of (2) and (3) is the hybrid Noether theorem.
\end{proof}
\begin{remark}
    Preservation of the connection is equivalent to the following diagram being commutative:
    \begin{equation}\label{diag:vertical}
        \begin{tikzcd}[row sep=huge,column sep=huge]
            & T^*E|_\mathcal{S} \arrow[r, "\tilde{\Delta}"] \arrow{dr}[near end]{J_\mathcal{S}} & T^*E \arrow[d, "J"]\\
            TE|_\mathcal{S} \arrow[r, "\Delta"] \arrow[dr, "\mathcal{A}_\mathcal{S}" below left] \arrow[ur, "\mathbb{F}L"] & TE \arrow[d, "\mathcal{A}"] \arrow[ur, crossing over, "\mathbb{F}L" above left, near start] & \mathfrak{g}^*\\
            & \mathfrak{g} \arrow[ur, "\mathbb{I}" below right]&
        \end{tikzcd}
    \end{equation}
\end{remark}
In some sense, the diagrams \eqref{diag:horizontal} and \eqref{diag:vertical} are opposites of one another. For interesting examples, we want $\mathcal{S}$ to be neither vertical nor horizontal.
\subsection{Interpretation for the Pendulum on a Cart}
The internal impact, $\mathcal{S} = \{\alpha\}\times\mathbb{R}$, is vertical. However, the external impact, $\mathcal{S} = \mathbb{S}^1\times\{z\}$ is \textit{not} horizontal. It can be seen, for example, that the impact map \eqref{eq:outer_cart} is not the identity on the $\dot{\theta}$ component which contradicts the conclusion of Theorem \ref{thm:hor}. For an impact to be horizontal, the impact surface must be a level-set of the function
\begin{equation}\label{eq:hor_cart}
    f(\theta,x) = \frac{m\ell}{M+m}\sin\theta + x.
\end{equation}
However, this function is the integral of the mechanical connection, i.e., $\mathcal{A} = 0$ implies that $f$ is constant. In order to have impacts, we need $\mathcal{A}\ne 0$. 

The impact map with impact condition \eqref{eq:hor_cart} is
\begin{equation*}
    \begin{split}
        p_\theta &\mapsto p_\theta + \varepsilon \frac{m\ell}{M+m}\cos\theta \\
        p_x &\mapsto p_x+\varepsilon
    \end{split}
\end{equation*}
with $\varepsilon = -2p_x$. 
This impact reverses the connection:
\begin{equation*}
    \mathcal{A}_{(\theta,x)}\left(\Delta(\dot\theta,\dot{x})\right) = -\mathcal{A}_{(\theta,x)}\left(\dot\theta,\dot{x}\right). 
\end{equation*}
Let $\alpha = \mathcal{A}_{(\theta_0,x_0)}(\dot{\theta}_0,\dot{x})$ be the value of the connection on the initial conditions. 
In velocity coordinates, we have
\begin{equation*}
    \begin{split}
        \dot\theta \mapsto \dot\theta, \quad 
        \dot{x} \mapsto \dot{x} - 2\alpha, \quad
        \alpha \mapsto -\alpha.
    \end{split}
\end{equation*}
The reversal of the connection is typical in low-dimensional examples as made clear in the following proposition.
\begin{proposition}
    For a $G$-impact system with $\dim(G)=1$ and $\mathcal{S}$ horizontal, we have $\Delta^*\mathcal{A} = -\mathcal{A}|_\mathcal{S}$.
\end{proposition}
\begin{proof}
    As $\dim(G) = 1$, $(T\mathcal{S})^\perp = \mathcal{V}$. Let $(q,\dot{q})\in TE|_\mathcal{S}$ be the state immediately before impact and $\mathcal{A}(q,\dot{q}) = \xi_E$ for $\xi\in\mathfrak{g}\cong\mathbb{R}$. In particular, $v^+ - v^- \propto \xi_E$. This provides us with
    \begin{equation*}
        \begin{split}
            \mathcal{A}(\dot{q}^+) &= \mathcal{A}(\dot{q}^-) + \mathcal{A}\left( -2\frac{\langle \dot{q}^-, \xi_E\rangle}{\langle \xi_E,\xi_E\rangle} \xi_E \right) \\
            &= \xi - 2\frac{\langle \dot{q}^-,\xi_E\rangle}{\langle \xi_E,\xi_E\rangle}\xi \\
            &= -\xi = -\mathcal{A}(\dot{q}^-).
        \end{split}
    \end{equation*}
\end{proof}
\section{Conclusions and Future Work}
 We have defined hybrid systems on principal bundles, studied the underlying geometry on the switching surface where impacts occur, and found conditions for which both exterior and interior impacts are preserved by the mechanical connection induced in the principal bundle. 

For future work, we wish to extend our analysis to explicitly time-dependent systems in the context of cosymplectic geometry. Note that given $\mathcal{S}\subset E\times\mathbb{R}$ of codimension 1 such that $\pi_t(\mathcal{S})\subset E$ is smooth for all $t$, time-dependent vertical remains vertical but that is not the case for horizontal.
In particular, one has the following observation that will be crucial in our further studies to extend this paper to the time-dependent situation: 
    let $S\subset E\times\mathbb{R}$ be a time-dependent impact surface. If $\pi_t(\mathcal{S})\subset E$ is vertical for all $t$, then the mechanical connection is preserved across impacts. 

\addtolength{\textheight}{-10cm}

\bibliographystyle{ieeetr}
\bibliography{references.bib}

\begin{thebibliography}{10}

\bibitem{lee_geometric_2013}
T.~Lee, K.~Sreenath, and V.~Kumar, ``Geometric control of cooperating multiple
  quadrotor {{UAVs}} with a suspended payload,'' in {\em 52nd {{IEEE
  Conference}} on {{Decision}} and {{Control}}}, ({Firenze}), pp.~5510--5515,
  {IEEE}, Dec. 2013.

\bibitem{westervelt_feedback_2018}
E.~R. Westervelt, J.~W. Grizzle, C.~Chevallereau, J.~H. Choi, and B.~Morris,
  {\em Feedback {{Control}} of {{Dynamic Bipedal Robot Locomotion}}}.
\newblock {Boca Raton}: {CRC Press}, Oct. 2018.

\bibitem{goebel_rafal_hybrid_2012}
R.~Goebel and R.~G. Sanfelice, {\em Hybrid {{Dynamical Systems}}}.
\newblock {Princeton University Press}, Mar. 2012.

\bibitem{schaft_introduction_2000}
A.~J. van~der Schaft and H.~Schumacher, {\em An Introduction to Hybrid
  Dynamical Systems}.
\newblock No.~251 in Lecture Notes in Control and Information Sciences,
  {London, New York}: {Springer}, 2000.

\bibitem{johnson_simple_1994}
S.~D. Johnson, ``Simple hybrid systems,'' {\em Int. J. Bifurcation Chaos},
  vol.~04, pp.~1655--1665, Dec. 1994.

\bibitem{ames_geometric_2007}
A.~D. Ames, R.~D. Gregg, E.~D.~B. Wendel, and S.~Sastry, ``On the {{Geometric
  Reduction}} of {{Controlled Three-Dimensional Bipedal Robotic Walkers}},'' in
  {\em 3rd {{IFAC Workshop}} on {{Lagrangian}} and {{Hamiltonian Methods}} for
  {{Nonlinear}}} (F.~Bullo and K.~Fujimoto, eds.), vol.~366, pp.~183--196,
  {Berlin}: {Springer}, 2007.

\bibitem{holmes_dynamics_2006}
P.~Holmes, R.~J. Full, D.~Koditschek, and J.~Guckenheimer, ``The {{Dynamics}}
  of {{Legged Locomotion}}: {{Models}}, {{Analyses}}, and {{Challenges}},''
  {\em SIAM Rev.}, vol.~48, pp.~207--304, Jan. 2006.

\bibitem{bloch2017quasivelocities}
A.~Bloch, W.~Clark, and L.~Colombo, ``Quasivelocities and symmetries in simple
  hybrid systems,'' in {\em 2017 IEEE 56th Annual Conference on Decision and
  Control (CDC)}, pp.~1529--1534, IEEE, 2017.

\bibitem{colombo2020symmetries}
L.~J. Colombo and M.~E.~E. Iraz{\'u}, ``Symmetries and periodic orbits in
  simple hybrid routhian systems,'' {\em Nonlinear Analysis: Hybrid Systems},
  vol.~36, p.~100857, 2020.

\bibitem{cortes_hamiltonian_2006}
J.~Cort{\'e}s and A.~M. Vinogradov, ``Hamiltonian theory of constrained
  impulsive motion,'' {\em J. Math. Phys.}, vol.~47, p.~042905, Apr. 2006.

\bibitem{cortes_mechanical_2001}
J.~Cort{\'e}s, M.~De~Le{\'o}n, D.~{Mart{\'i}n de Diego}, and S.~Mart{\'i}nez,
  ``Mechanical systems subjected to generalized non-holonomic constraints,''
  {\em Proc. R. Soc. Lond. A}, vol.~457, pp.~651--670, Mar. 2001.

\bibitem{ibort_geometric_1998}
A.~Ibort, M.~de~Le{\'o}n, E.~A. Lacomba, J.~C. Marrero, D.~M. de~Diego, and
  P.~Pitanga, ``Geometric formulation of mechanical systems subjected to
  time-dependent one-sided constraints,'' {\em J. Phys. A: Math. Gen.},
  vol.~31, pp.~2655--2674, Mar. 1998.

\bibitem{ibort_mechanical_1997}
A.~Ibort, M.~de~Le{\'o}n, E.~A. Lacomba, D.~M. de~Diego, and P.~Pitanga,
  ``Mechanical systems subjected to impulsive constraints,'' {\em J. Phys. A:
  Math. Gen.}, vol.~30, pp.~5835--5854, Aug. 1997.

\bibitem{BL}
A.~Bloch, {\em Nonholonomic Mechanics and Control}.
\newblock Interdisciplinary Applied Mathematics Series, 24, Springer-Verlag,
  2003.

\bibitem{colombo2021generalized}
L.~J. Colombo, M.~De~Le{\'o}n, M.~E.~E. Iraz{\'u}, and A.~L{\'o}pez-Gord{\'o}n,
  ``Generalized hybrid momentum maps and reduction by symmetries of forced
  mechanical systems with inelastic collisions,'' {\em arXiv preprint
  arXiv:2112.02573}, 2021.

\bibitem{mars}
J.~E. Marsden and T.~S. Ratiu, {\em Introduction to mechanics and symmetry: a
  basic exposition of classical mechanical systems}, vol.~17.
\newblock Springer Science \& Business Media, 2013.

\bibitem{kirk2004optimal}
D.~E. Kirk, {\em Optimal control theory: an introduction}.
\newblock Courier Corporation, 2004.

\bibitem{brogliato1999nonsmooth}
B.~Brogliato and B.~Brogliato, {\em Nonsmooth mechanics}, vol.~3.
\newblock Springer, 1999.

\bibitem{HolSch09}
D.~Holm, T.~Schmah, and C.~Stoica, {\em Geometric Mechanics and Symmetry}.
\newblock Oxford University Press, 2009.

\bibitem{kobayashi1996foundations}
S.~Kobayashi and K.~Nomizu, {\em Foundations of Differential Geometry, Volume
  2}, vol.~61.
\newblock John Wiley \& Sons, 1996.

\bibitem{ames_hybrid_2006}
A.~Ames and S.~Sastry, ``Hybrid {{Routhian}} reduction of {{Lagrangian}} hybrid
  systems,'' in {\em 2006 {{American Control Conference}}}, p.~6 pp., June
  2006.

\end{thebibliography}

\end{document}